\documentclass[letterpaper, 10 pt, conference]{ieeeconf}

\IEEEoverridecommandlockouts
\usepackage{graphics}
\usepackage{amsmath}
\usepackage{amssymb}
\usepackage{epsfig}
\usepackage{mathptmx}
\usepackage{times}
\usepackage{algorithm}
\usepackage{algorithmic}

\makeatletter
\newcommand{\Statex}{\item[]}
\makeatother

\usepackage{mathtools}

\usepackage{amsthm}
\usepackage{hyperref}
\usepackage{eucal}
\usepackage{balance}

\newtheorem{theorem}{Theorem}
\newtheorem{corollary}{Corollary}
\newtheorem{definition}{Definition}
\newtheorem{lemma}{Lemma}
\newtheorem{proposition}{Proposition}
\newtheorem{assumption}{Assumption}

\title{\LARGE \bf
Model-Based Learning of Near-Optimal Finite-Window Policies in POMDPs
}

\author{Philip Jordan and Maryam Kamgarpour
\thanks{Philip Jordan and Maryam Kamgarpour are with SYCAMORE group,
Institute of Mechanical Engineering, EPFL Switzerland.
        $\left\{ \text{\texttt{philip.jordan, maryam.kamgarpour}} \right\}$\texttt{@epfl.ch}}}

\begin{document}

\maketitle
\thispagestyle{empty}
\pagestyle{empty}

\begin{abstract}
We study model-based learning of finite-window policies in tabular partially observable Markov decision processes (POMDPs). A common approach to learning under partial observability is to approximate unbounded history dependencies using finite action-observation windows. This induces a finite-state Markov decision process (MDP) over histories, referred to as the superstate MDP. Once a model of this superstate MDP is available, standard MDP algorithms can be used to compute optimal policies, motivating the need for sample-efficient model estimation. Estimating the superstate MDP model is challenging because trajectories are generated by interaction with the original POMDP, creating a mismatch between the sampling process and target model. We propose a model estimation procedure for tabular POMDPs and analyze its sample complexity. Our analysis exploits a connection between filter stability and concentration inequalities for weakly dependent random variables. As a result, we obtain tight sample complexity guarantees for estimating the superstate MDP model from a single trajectory. Combined with value iteration, this yields approximately optimal finite-window policies for the~POMDP.
\end{abstract}

\section{Introduction}

Partially observable Markov decision processes (POMDPs) are a fundamental framework for modeling sequential decision-making under uncertainty. While Markov decision processes (MDPs) have become the standard model in reinforcement learning, they rely on the agent having full knowledge of the environment's state -- an idealization that is often unrealistic in practice. POMDPs relax this requirement by allowing the agent to receive only partial or noisy observations of an underlying state, making them applicable to a broader class of problems, such as robotic control~\cite{lauri2022partially} and autonomous driving~\cite{bai2015intention} under sensor noise, or strategic interaction in games of imperfect information~\cite{yao2020solving}. Developing principled methods for learning in POMDPs is therefore a problem of fundamental importance.

However, it has long been recognized that the greater generality of POMDPs comes with inherent hardness barriers, arising from the fact that optimal policies may depend on the full history of observations. Apart from very specific settings such as linear quadratic Gaussian problems, even the planning task of identifying a near-optimal policy under full knowledge of the model is computationally intractable~\cite{papadimitriou1987complexity}. Similar obstacles arise in the learning problem: even ignoring computational constraints, finding a near-optimal policy in a POMDP requires an exponential number of samples~\cite{krishnamurthy2016pac}.
As a result, a large body of work has explored approximate solution methods that trade optimality for computational tractability or sample-efficiency. In this work, we focus on approximate approaches that are implementable in a computationally efficient manner. One such approach is the use of finite-window techniques, where policies depend only on a fixed-length history window of recent actions and observations~\cite{mccallum1996reinforcement,loch1998using,meuleau1999solving}.

Recently, several works have sought to provide a theoretical explanation for the empirical success of finite-window policies~\cite{kara2023convergence,cayci2024finite,golowich2022learning,anjarlekar2026scalable}. A standard approach in these works is to approximate the underlying POMDP by an MDP, referred to as a \emph{superstate MDP}, whose states correspond to windows of recent history. This reduction from POMDP to superstate MDP provides a principled framework for designing provable algorithms for POMDPs. Under appropriate structural assumptions, one can show that optimal finite-window policies for the superstate MDP are also near-optimal for the respective POMDP \cite{kara2022near}.
More broadly, the superstate MDP is an instance of an MDP with \emph{approximate information state}, a framework which in~\cite{subramanian2022approximate} has been applied to planning and learning in partially observed systems.

A central challenge in theoretical analysis of the superstate MDP approach is that, although we can define Markovian transition dynamics over superstates, trajectories can only be sampled from the underlying POMDP, whose dynamics depend on the full history of actions and observations. In this work, we consider a model-based approach for tabular POMDPs, aiming to estimate the reward function and transition kernel of the superstate MDP from trajectories generated by the POMDP. The key difficulty is to learn an optimal policy despite this mismatch between the Markovian structure of the target model and the history-dependent nature of the data-generating process. To the best of our knowledge, how to accomplish this in a sample-efficient manner has not been addressed in prior work.

Prior work~\cite{kara2023convergence} analyzes Q-learning in the superstate MDP and derives approximation guarantees based on filter stability; however, the resulting convergence guarantees are only asymptotic.
A quasi-polynomial finite-time complexity guarantee for a finite-window approach is established in~\cite{golowich2022learning}, but no explicit sample complexity bound for a fixed window size is provided. More recently,~\cite{cayci2024finite,anjarlekar2026scalable} study the finite-time convergence of temporal-difference-learning-based policy evaluation in the superstate MDP under linear function approximation. Applied to the tabular setting, the resulting guarantees imply sample complexity bounds for learning an $\epsilon$-optimal policy on the order of~$\mathcal{O}(\epsilon^{-4})$, leaving a substantial gap to the tight $\mathcal{O}(\epsilon^{-2})$ rates known for fully observable MDPs~\cite{azar2012sample}. Even in the basic tabular case, it is thus not well-understood whether this sample complexity gap is avoidable or the fundamental price of partial observability.

\noindent\textbf{Contributions.} Motivated by the above, we consider tabular POMDPs, and our contributions are summarized as follows.
\begin{enumerate}
\item We propose a method for estimating the model of a superstate MDP based on a trajectory sampled from the underlying POMDP. Transition probabilities are estimated from the empirical frequencies of the observed finite action-observation windows, and rewards from the observed observation-action pairs.
\item Under regularity assumptions on the transition and observation kernels, see Assumptions~\ref{ass:uni-min} and~\ref{ass:uni-obs}, we provide tight sample complexity guarantees for the proposed superstate model estimation, see Proposition~\ref{prop:est-prob}.
\item Through value iteration on the estimated model, we derive an efficient algorithm for learning policies with $m$-step history dependence. In particular, for a trajectory of length $\mathcal{O}(\epsilon^{-2})$, the obtained policy is $\epsilon$-optimal for the POMDP up to an additional error due to finite-window approximation, see Theorem~\ref{thm:main}.
\end{enumerate}

\section{Preliminaries}

\noindent\textbf{Notation.} Let $\mathbb{N} \coloneqq \left\{ 1,2,\dots \right\}$. For $n \in \mathbb{N}$, let $[n] \coloneqq \left\{ 1,2,\dots,n \right\}$. For $a,b \in \mathbb{N}$ with $a \leq b$, let $[a,b] \coloneqq \left\{ a,a+1,\dots,b-1,b \right\}$.

\subsection{Problem Formulation}

\noindent\textbf{POMDP.} We consider tabular infinite-horizon discounted POMDPs defined as a tuple $\mathcal{P}=(\mathcal{S},\mathcal{A},\mathcal{O},r,\mathbb{P},\mathbb{O},\mu,\gamma)$ with finite state space $\mathcal{S}$ of size $S \coloneqq \vert \mathcal{S} \vert$, finite action space $\mathcal{A}$ of size $A\coloneqq |\mathcal{A}|$, and finite observation space $\mathcal{O}$ of size $O \coloneqq |\mathcal{O}|$. The reward function is defined\footnote{In some related works, $r$ is defined as a function of state and action. We use the observation-action definition, since state-dependent rewards may reveal additional information about the hidden state. Our results also generalize to state-action rewards.} as $r:\mathcal{O} \times \mathcal{A} \to [-1,1]$, and transition probabilities are given by $\mathbb{P}(\cdot \mid s,a) \in \Delta(\mathcal{S})$ for each $s \in \mathcal{S}$ and $a \in \mathcal{A}$, where $\Delta(\mathcal{S})$ denotes the probability simplex over~$\mathcal{S}$. For each $s \in \mathcal{S}$, observation probabilities are given by $\mathbb{O}(\cdot \mid s) \in \Delta(\mathcal{O})$.

At time $t \in \mathbb{N}$, the agent receives observation $o_t \sim \mathbb{O}(\cdot \mid s_t)$, chooses an action $a_t \in \mathcal{A}$, obtains reward $r(o_t,a_t)$, and the unobserved state transitions to $s_{t+1} \sim \mathbb{P}(\cdot \mid s_t,a_t)$. The initial state $s_1$ is drawn from the distribution $\mu \in \Delta(\mathcal{S})$. We consider infinite interactions with rewards discounted by~$\gamma \in (0,1)$.

\textbf{Histories and policies.} For $n \in \mathbb{N}$, we define $\mathcal{H}^n \coloneqq (\mathcal{O} \times \mathcal{A})^{n-1} \times \mathcal{O}$ as the set of $n$-step action-observation histories. Let $\mathcal{H}^{\leq m} \coloneqq \bigcup_{n \in [m]} \mathcal{H}^n$ and $\mathcal{H}^{\infty} \coloneqq \bigcup_{m \in \mathbb{N}} \mathcal{H}^{\leq m}$.
For $t \in \mathbb{N}$, let~$h \in \mathcal{H}^t$ be a history written as $h=(o_1,a_1,o_2,\dots,a_{t-1},o_t)$. For integers $k \leq l$, we define the sub-history $h_{k:l}=(o_k,a_k,\dots,a_{l-1},o_l)$, restricted to the indices contained in~$[1,t]$. In particular, $h_{1:t}=h$. We denote by $|h| = t$ the length of $h$, that is, the number of observations it contains.
We distinguish between the class of history-dependent deterministic policies~$\Pi^{\infty} \coloneqq \left\{ \pi=(\pi_t)_{t \in \mathbb{N}} \mid \pi_t:\mathcal{H}^{t} \to \mathcal{A} \text{ for all } t \in \mathbb{N} \right\}$, and, for $m \in \mathbb{N}$, the class of $m$-step history-dependent stationary policies~$\Pi^m \coloneqq \left\{ \pi:\mathcal{H}^{\leq m} \to \mathcal{A} \right\}$. At time $t \in \mathbb{N}$, given full history $h \in \mathcal{H}^{t}$, a policy $\pi \in \Pi^{\infty}$ chooses $a = \pi_t(h)$, whereas a policy~$\pi^m \in \Pi^m$ chooses $a = \pi^m(h_{t-m+1 \,:\, t})$.
For any $\pi \in \Pi^{\infty} \cup \Pi^m$, we can define the discounted value function
\begin{align*}
V(\pi) \coloneqq \mathbb{E}_{\pi,s_1 \sim \mu} \left[ \sum_{t=1}^{\infty} \gamma^{t-1} r(o_t,a_t) \right].
\end{align*}
The expectation is taken over the distribution of trajectories induced by the POMDP $\mathcal{P}$ under ($m$-step or fully history-dependent) policy $\pi$, including the randomness from state transitions and observation emissions.

Given some~$\epsilon > 0$ and $m \in \mathbb{N}$, our objective is to learn a near-optimal $m$-step window policy $\pi^m \in \Pi^m$ such that 
\begin{align*}
V^{\star}-V(\pi^m) \leq \epsilon \quad \text{where} \quad V^{\star} \coloneqq \max_{\pi \in \Pi^{\infty}}V(\pi).
\end{align*}

\subsection{MDP Representations of a POMDP}
\label{sec:mdp-repr}

Our approach builds on the well-known equivalence between POMDPs and fully observable MDPs with history-dependent states~\cite{puterman1994markov,kaelbling1998planning}. In particular, any POMDP $\mathcal{P}$ can be cast as an MDP with state space given by the set of full histories, resulting in an equivalent infinite-state MDP $\mathcal{M}^\infty$.

Since solving $\mathcal{M}^\infty$ is intractable due to its infinite state space, a common approach is to approximate it with a finite-state MDP~\cite{kara2023convergence,cayci2024finite,golowich2022learning}. In this work, we consider a truncation-based approximation, namely the $m$-step superstate MDP $\mathcal{M}^m$, whose states consist of history windows of length at most $m$.

\textbf{History-state MDP.} This MDP $\mathcal{M}^{\infty}$ has action space $\mathcal{A}$ and (infinite) state space $\mathcal{H}^{\infty}$. Rewards $r^{\infty}$ and transition probabilities $\mathbb{P}^{\infty}$ are defined as follows. For each $a \in \mathcal{A}$, and $h=(o_1,a_1,\dots,a_{t-1},o_t) \in \mathcal{H}^{\infty}$, we let $r^{\infty}(h,a)\coloneqq r(o_t,a)$. Further, for each $a \in \mathcal{A}$ and $h \in \mathcal{H}^{\infty}$, $\mathbb{P}^{\infty}(h \circ (a,o) \mid h,a)$ is defined as the probability of observing $o \in \mathcal{O}$ at the next time step in the underlying POMDP $\mathcal{P}$, conditioned on the history $h$ and the action~$a$.
The resulting decision process is Markovian by construction.
Moreover, the history-state MDP $\mathcal{M}^{\infty}$ is equivalent to the underlying POMDP $\mathcal{P}$ in the sense that a policy $\pi \in \Pi^{\infty}$ is optimal for $\mathcal{M}^{\infty}$ if and only if it is optimal for $\mathcal{P}$.
However, the infinite state space $\mathcal{H}^{\infty}$ makes this representation impractical, motivating the finite-state approximation $\mathcal{M}^m$.

\textbf{Superstate MDP.} 
For some $m \in \mathbb{N}$, the $m$-step superstate MDP $\mathcal{M}^m$ has action space $\mathcal{A}$ and (finite) state space $\mathcal{H}^{\leq m}$.
For each $a \in \mathcal{A}$ and $w=(o_1,a_1,\dots,a_{|w|-1},o_{|w|}) \in \mathcal{H}^{\leq m}$, rewards are defined as $r^m(w,a)\coloneqq r(o_{|w|},a)$.
To define the transition probabilities $\mathbb{P}^m(\cdot \mid w,a) \in \Delta(\mathcal{H}^{\leq m})$, we introduce beliefs.
Given full history $h \in \mathcal{H}^{\infty}$, let $b(\cdot \mid h) \in \Delta(\mathcal{S})$ denote the posterior distribution over states in $\mathcal{P}$.
Since a superstate $w$ is not a sufficient statistic for the history, we define the belief $b(s \mid w)$ as the probability of being in state $s$ at the end of $w$ from initial distribution\footnote{Instead of $\mu$, we could pick as prior an arbitrary distribution over $\mathcal{S}$.} $\mu$ along any state sequence that is consistent with $w$,
\begin{align*}
b(s \mid w) &=
\frac{\mathbb{O}(o_{|w|} \mid s)}{Z(w)}
\!\!\! \sum_{s_1,\dots,s_{|w|-1}} \!\!\! \mu(s_1) \\
&\qquad\qquad\quad \cdot
\Biggl[ \prod_{k=1}^{|w|-1} \mathbb{O}(o_k \mid s_k) \mathbb{P}(s_{k+1} \mid s_k, a_k) \Biggr]
\end{align*}
where $s_{|w|} \coloneqq s$ and the normalization factor $Z(w)$ is the sum over $s \in \mathcal{S}$ of the unnormalized probabilities of reaching $s$.

For all $w,w^{\prime} \in \mathcal{H}^{\leq m}$, $a \in \mathcal{A}$, and $o \in \mathcal{O}$ with $w' = w_{|w|-m+2 \,:\, |w|} \circ (a,o)$, where $\circ$ denotes concatenation,
\begin{align}
\label{eqn:Pm-def}
\mathbb{P}^m(w' \mid w,a)
\;\coloneqq\;
\sum_{s,s^{\prime} \in \mathcal{S}}
b(s \mid w)\, \mathbb{P}(s^{\prime} \mid s,a) \, \mathbb{O}(o \mid s^{\prime}).
\end{align}
Otherwise, if $w^{\prime}$ cannot be obtained by concatenating $w$ with any action-observation pair, we set $\mathbb{P}^m(w' \mid w,a) = 0$. The resulting decision process is Markovian by construction, as transitions depend only on the current superstate $w$. We point out that the superstate MDP is a conceptual object: the agent interacts with $\mathcal{P}$ and does not have access to $\mathbb{P}$ or $\mathbb{O}$; thus beliefs cannot be computed and $\mathcal{M}^m$ is used only for analysis.

\section{Model-Based Learning Algorithm}

In this section, we present our proposed method, see Algorithm~\ref{alg:main}, for learning near-optimal finite-window policies in POMDPs. We begin by outlining the high-level approach and the associated challenges, followed by a description of the sampling procedure and model estimation phase. Then, we apply value iteration to the resulting approximate model to compute a near-optimal policy.

\subsection{Approach and Challenges}

The superstate MDP $\mathcal{M}^m$ is a finite-state MDP, and thus, by standard results on Markov decision processes, it admits an optimal stationary policy $\pi^m_\star \in \Pi^m$, corresponding to an $m$-step history-dependent policy in the original POMDP~$\mathcal{P}$. Our approach is to estimate a model of $\mathcal{M}^m$ from a trajectory generated by interacting with $\mathcal{P}$, and then compute a policy~$\pi^m$ by finding the optimal policy of this estimated~MDP.
As will be shown in our Theorem~\ref{thm:main}, $\pi^m$ achieves nearly optimal value in the original POMDP.

Such model-based approaches have been commonly studied in the fully observable MDP setting, and the obtained sample complexity for learning approximately optimal policies matches the respective lower bound \cite{azar2012sample}. In our partially observable case, this approach gives rise to two main challenges:
\begin{enumerate}
\item\label{item:ch1} \textbf{Model estimation.} When collecting a trajectory through interaction with $\mathcal{P}$, the history transitions are sampled according to $\mathbb{P}^{\infty}$, not the superstate transition kernel $\mathbb{P}^m$. Sampling from $\mathbb{P}^m$ would require access to the state distribution induced by beliefs based on $m$-step history windows, see (\ref{eqn:Pm-def}), whereas the true transitions depend on the entire history. How can one estimate a model of $\mathcal{M}^m$ in a sample-efficient manner despite this mismatch?
\item\label{item:ch2} \textbf{Planning under approximation.} Given an approximate model of $\mathcal{M}^m$, how can one compute a policy whose value in the original POMDP $\mathcal{P}$ is near-optimal?
\end{enumerate}
We propose a superstate model estimation procedure that provides tight confidence bounds under regularity conditions on transition and observation kernels, see Assumptions~\ref{ass:uni-min} and~\ref{ass:uni-obs}, thus resolving Challenge~\ref{item:ch1}. By combining this estimation with MDP planning methods such as value iteration, under the same assumptions, we obtain a policy whose value in $\mathcal{P}$ is near-optimal, thereby addressing Challenge~\ref{item:ch2}.

\subsection{Sampling \& Model Estimation Phase}

We now outline the first phase of Algorithm~\ref{alg:main}. We assume the agent can only interact with the POMDP $\mathcal{P}$ and does not have access to a generative model or the underlying transition kernel. In this setting, we collect a trajectory~$\tau=(o_1,a_1,r_1,\dots,o_T,a_T,r_T)$ of length $T$ by selecting actions uniformly at random, i.e., $a_t \sim \mathcal{U}(\mathcal{A})$ for all $t \in [T]$. Choosing actions uniformly ensures sufficient exploration of the superstate space, which is critical for constructing accurate estimates. Using the trajectory~$\tau$, we define a model estimate for the superstate MDP $\mathcal{M}^m$, where transition probabilities are obtained from the empirical frequencies of the respective $m$-step window transitions. This approach is inspired by model-based methods for fully observable MDPs \cite{azar2012sample,kearns2002near}, where the model is estimated from empirical frequencies over state-action trajectories.

\textbf{Transition probabilities.} For all $w,w^{\prime} \in \mathcal{H}^{\leq m}$ and $a \in \mathcal{A}$, take the empirical average
\begin{align}
\hat{\mathbb{P}}^m(w^{\prime} \mid w,a) \coloneqq \label{eq:sample-p}
\begin{cases}
\frac{\sum_{t=1}^{T} \mathbf{1}\{ \tau_{t-|w^{\prime}|+2\,:\,t+1}=w^{\prime} \;\land\; \tau_{t-|w|+1\,:\,t}=w \;\land\; a_t=a \}}{\sum_{t=1}^{T} \mathbf{1}\{ \tau_{t-|w|+1\,:\,t}=w \land a_t=a \}}, \\[6pt]
\quad \text{if } \sum_{t=1}^{T} \mathbf{1}\{ \tau_{t-|w|+1\,:\,t}=w \land a_t=a \} \ge 1 \\
\qquad\qquad\quad \text{ and } |w^{\prime}| = \min(m, |w|+1), \\[6pt]
0, \; \text{otherwise}.
\end{cases}
\end{align}

\noindent\textbf{Rewards.} Note that the reward function $r$ is not available to the agent, which observes only the realized rewards $r_t$ along the trajectory; thus $r^m$ needs to be estimated as well. However, since $r$ is a deterministic function of observation and action, and the superstate contains the current observation, $r^m$ is also a deterministic function of $(w,a)$. Hence, the estimate $\hat{r}^m(w,a)$ is exact after a single visit of the respective observation-action pair, and we define for all $w=(\tilde{o}_1,\tilde{a}_1,\dots,\tilde{a}_{|w|-1},\tilde{o}_{|w|}) \in \mathcal{H}^{\leq m}$ and $a \in \mathcal{A}$,
\begin{align}
\hat{r}^m(w,a) \coloneqq \label{eq:sample-r}
\begin{cases}
r_{\hat{t}} &\text{if } \exists\, \hat{t} \in [T] \text{ s.t. } o_{\hat{t}}=\tilde{o}_{|w|} \;\land\; a_{\hat{t}}=a, \\[6pt]
0, &\text{otherwise}.
\end{cases}
\end{align}

\subsection{Value Iteration Phase}

We next describe the second phase of Algorithm~\ref{alg:main}. Let $Q^m_{\star} \in \mathbb{R}^{|\mathcal{H}^{\leq m}| \times |\mathcal{A}|}$ denote the optimal action-value function of the $m$-step superstate MDP $\mathcal{M}^m$, which satisfies the Bellman optimality equation for all $w \in \mathcal{H}^{\leq m}$ and $a \in \mathcal{A}$:
\begin{align*}
Q^m_{\star}(w,a)
= r^m(w,a)
+ \gamma \sum_{w' \in \mathcal{H}^{\leq m}}
\mathbb{P}^m(w' \mid w,a)
\max_{a' \in \mathcal{A}} Q^m_{\star}(w',a').
\end{align*}
Using the estimated model $(\hat{\mathbb{P}}^m, \hat{r}^m)$, we approximate $Q^m_{\star}$ via value iteration. For $k \in [K]$, let $\hat{Q}_k \in \mathbb{R}^{|\mathcal{H}^{\leq m}| \times |\mathcal{A}|}$ denote the $k$-th iterate, see Line~\ref{ln:opt-update} of Algorithm~\ref{alg:main}, with $\hat{Q}_0$ initialized to $\mathbf{0}$. Each value iteration update is performed using the Bellman optimality operator for the estimated model.  

After $K$ iterations, the final iterate $\hat{Q}_K$ defines an approximately optimal policy $\pi^m \in \Pi^m$ for the superstate MDP $\mathcal{M}^m$ by acting greedily with respect to $\hat{Q}_K$. The suboptimality of~$\pi^m$ can be attributed to (a) the finite-window approximation, (b) the model estimation error, and (c) using a finite number of value iteration updates; each of these errors will be bounded in the proof of our guarantee, Theorem~\ref{thm:main}.

\begin{algorithm}
\caption{Learning Finite-Window Policies in POMDPs}
\label{alg:main}
\begin{algorithmic}[1]
\STATE \textbf{Input:} parameters $T$ and $K$, window length~$m$.
\vspace{.5em}
\Statex \COMMENT{\textbf{Phase 1:} Sampling \& model estimation}
\vspace{.5em}
\STATE Collect a trajectory $\tau=(o_1,a_1,r_1,\dots,o_T,a_T,r_T)$ of length~$T$ by choosing actions uniformly at random, i.e., $a_t \sim \mathcal{U}(\mathcal{A})$ for all $t \in [T]$.
\vspace{.5em}
\STATE Estimate $\hat{\mathbb{P}}^m$ and $\hat{r}^m$ as in~(\ref{eq:sample-p}) and~(\ref{eq:sample-r}), respectively.
\vspace{.5em}
\Statex \COMMENT{\textbf{Phase 2:} Value iteration}
\vspace{.5em}
\STATE Let $\hat{Q}_0 \coloneqq \mathbf{0} \in \mathbb{R}^{|\mathcal{H}^{\leq m}| \times |\mathcal{A}|}$.
\vspace{.4em}
\FOR{$k = 1,2,\dots,K$}
\FOR{$w \in \mathcal{H}^{\leq m}, a \in \mathcal{A}$}
\STATE \(
\begin{aligned}
\hat{Q}_k(w,a) \coloneqq &\, \hat{r}^m(w,a) \\
&+ \gamma \sum_{w^{\prime} \in \mathcal{H}^{\leq m}}
\hat{\mathbb{P}}^m(w^{\prime} \mid w,a)
\max_{a' \in \mathcal{A}} \hat{Q}_{k-1}(w^{\prime},a').
\end{aligned}
\) \label{ln:opt-update}
\ENDFOR
\ENDFOR
\vspace{.3em}
\STATE Let $\pi^m(w) \in \arg\max_{a \in \mathcal{A}} \hat{Q}_K(w,a)$ for all $w \in \mathcal{H}^{\leq m}$.
\STATE \textbf{Output:} policy $\pi^m \in \Pi^m$.
\end{algorithmic}
\end{algorithm}

\subsection{Assumptions \& Guarantees}
\label{sec:ass}

In the following, we first state our assumptions, and then present our main result, Theorem~\ref{thm:main}, which provides a sample complexity guarantee for Algorithm~\ref{alg:main}.
\begin{assumption}
\label{ass:uni-min}
There exists $\alpha > 0$, such that for all $s,s^{\prime} \in \mathcal{S}$ and $a \in \mathcal{A}$,
\begin{align*}
\mathbb{P}(s^{\prime} \mid s,a) \geq \alpha.
\end{align*}
\end{assumption}
Assumption~\ref{ass:uni-min} ensures that, under any policy, the induced Markov chain over hidden states is uniformly ergodic~\cite{meyn2012markov}. This property is important for our analysis, as it enables the application of concentration results for hidden Markov models~\cite{kontorovich2008concentration} based on a single trajectory\footnote{We note that a weaker version of the condition -- requiring only $n$-step transition probabilities, for some $n \in \mathbb{N}$, to be uniformly lower bounded -- would also suffice; we adopt the one-step formulation for clarity of analysis.}.

Additionally, our model estimation guarantees require all $m$-step history windows to be sufficiently explored. To achieve this, we require observations to be sufficiently noisy.
\begin{assumption}
\label{ass:uni-obs}
There exists $\beta > 0$ such that for all $s \in \mathcal{S}$ and $o \in \mathcal{O}$,
\begin{align*}
\mathbb{O}(o \mid s) \geq \beta.
\end{align*}
\end{assumption}

Next, we state our sample complexity guarantee for Algorithm~\ref{alg:main} in learning near-optimal finite-window policies.
\begin{theorem}
\label{thm:main}
Let Assumptions~\ref{ass:uni-min} and~\ref{ass:uni-obs} hold. For $\delta > 0$, $m \in \mathbb{N}$, and $\epsilon \geq 2(1-\rho)^{m-1}$ with $\rho = S \alpha \beta$, suppose we run Algorithm~\ref{alg:main} with
\begin{align}
\label{eqn:main-thm-T}
T &\geq \frac{32 (2+\epsilon)^2 A^{2m} (m+1)^2}{\alpha^2 S^2 \beta^{2m} \epsilon^2} \log \left( \frac{48\,A^{2m+1}O^{2m}}{\delta} \right),\\[4pt]
K &\geq \frac{\log (2(1-\gamma) / \epsilon)}{1-\gamma}.
\end{align}
Then with probability $1-\delta$, the output policy $\pi^m$ satisfies
\begin{align*}
V^{\star} - V(\pi^m) \leq
\underbrace{\frac{3\epsilon}{(1-\gamma)^2}}_{\substack{\text{estimation} \\ \text{error}}}
+ \underbrace{\frac{4(1-\rho)^{m-1}}{(1-\gamma)^2}}_{\substack{\text{approximation} \\ \text{error}}} \leq \frac{5\epsilon}{(1-\gamma)^2}.
\end{align*}
\end{theorem}
\noindent\textbf{Remarks.}
\begin{enumerate}
\item The \emph{approximation error} is the price of restricting to $m$-step window policies: it is independent of the collected data and decays only as the window length~$m$ grows. The \emph{estimation error} collects the model estimation and value iteration errors, and is controlled through the trajectory length~$T$ and the number of value iterations~$K$. It can, however, only be driven down to the floor imposed by $\epsilon \geq 2(1-\rho)^{m-1}$, below which our model estimation guarantee, Proposition~\ref{prop:est-prob}, no longer applies. Attaining a target accuracy hence requires choosing a sufficiently large window length as well; see Corollary~\ref{cor:main} below.

\item The exponential dependence on $m$ in our bound is also present in the works \cite{cayci2024finite,golowich2022learning,anjarlekar2026scalable}.
For a fixed, sufficiently large window length~$m$, our method improves the $\epsilon$-dependence over prior work: when specialized to our setting of a tabular POMDP with a sliding-window approximation, the results of \cite{cayci2024finite} and \cite{anjarlekar2026scalable} yield a sample complexity of $\mathcal{O}(\epsilon^{-4})$, whereas ours achieves $\mathcal{O}(\epsilon^{-2})$.

\item In the standard MDP setting, the presented model-based approach is analyzed in~\cite{azar2012sample} and achieves a rate of~$\mathcal{O}(\epsilon^{-2})$, along with a matching lower bound. Thus, up to the additional~$\mathcal{O}((1-\rho)^{m-1} / (1-\gamma)^2)$ error introduced by the finite-window approximation, our method matches the MDP case, and the lower bound, showing that the $\epsilon$-dependence in our result is tight.
\end{enumerate}

Theorem~\ref{thm:main} treats the window length $m$ as fixed. Given a target accuracy $\bar{\epsilon}>0$, the following corollary chooses $m$ as a function of $\bar{\epsilon}$, and states the resulting overall sample complexity.
\begin{corollary}
\label{cor:main}
Let Assumptions~\ref{ass:uni-min} and~\ref{ass:uni-obs} hold, let $\bar{\epsilon}>0$ and $\delta>0$, and set
\begin{align*}
\epsilon = \frac{(1-\gamma)^2 \bar{\epsilon}}{5},
\quad
m = 1+\left\lceil \frac{\log \big( 10 \, (1-\gamma)^{-2} \bar{\epsilon}^{-1} \big)}{\log \left( 1/(1-\rho) \right)} \right\rceil.
\end{align*}
Suppose we run Algorithm~\ref{alg:main} with $T$ and $K$ as in Theorem~\ref{thm:main}. Then with probability $1-\delta$, the output policy satisfies $V^{\star}-V(\pi^m) \leq \bar{\epsilon}$. Moreover, abbreviating $\eta \coloneqq 2\log (A/\beta) / \log \left( 1/(1-\rho) \right)$, the requirement~(\ref{eqn:main-thm-T}) on $T$ is met by
\begin{align*}
T = \widetilde{\mathcal{O}} \Bigg( &\frac{A^4}{\alpha^2 S^2 \beta^4 \rho^2}
\left( \frac{c}{\bar{\epsilon}\,(1-\gamma)^2} \right)^{2+\eta} \left( \frac{\log (AO)}{\rho} + \log \frac{1}{\delta} \right) \Bigg)
\end{align*}
for an absolute constant $c>0$, where $\widetilde{\mathcal{O}}$ hides logarithmic factors in $1/\bar{\epsilon}$ and $1/(1-\gamma)$.
\end{corollary}
\begin{proof}
The stated choice of $m$ ensures $\epsilon \geq 2(1-\rho)^{m-1}$, so that Theorem~\ref{thm:main} applies and yields $V^{\star}-V(\pi^m) \leq 5\epsilon/(1-\gamma)^2 = \bar{\epsilon}$; the expression for $T$ follows by substitution into~(\ref{eqn:main-thm-T}).
\end{proof}
\noindent In Corollary~\ref{cor:main}, the excess exponent $\eta$ is the price of the exponential dependence on $m$ discussed above; it vanishes as the filter-stability rate $\rho$ tends to one, and degrades as $\rho$ tends to zero. In comparison, note that with the $\epsilon^{-4}$ rate of \cite{cayci2024finite,anjarlekar2026scalable}, one obtains exponent $4+\eta$ instead of $2+\eta$.

\section{Proof of Theorem~\ref{thm:main}}

We now present our analysis. First, we introduce filter stability, a key property for obtaining approximation guarantees in finite-window methods. Based on this, we establish confidence bounds for the sampling and model estimation phase of Algorithm~\ref{alg:main}. We then bound the optimality gap of the resulting policy after the value iteration phase relative to the optimal POMDP value $V^{\star}$.

\subsection{Filter Stability}

Filter stability is a condition that has been used in prior work on finite-history approximation in hidden Markov models~\cite{van2008hidden} and POMDPs~\cite{kara2023convergence,anjarlekar2026scalable}. We establish this property under our assumptions introduced in Section~\ref{sec:ass}. Filter stability states that the distance between two beliefs induced by any two histories strictly decreases after one step of interaction with the POMDP. Hence, the condition implies that the error due to an incorrect initial belief decays exponentially over time. For two distributions $\nu,\nu^{\prime} \in \Delta(\mathcal{S})$, their distance in total variation is $\|\nu-\nu^{\prime}\|_{TV} \coloneqq \frac{1}{2} \sum_{s\in\mathcal{S}} |\nu(s)-\nu^{\prime}(s)|$.
\begin{definition}
\label{def:filter-stab}
Filter stability holds for a POMDP $\mathcal{P}$ if there exists $\rho > 0$ such that for any $h,h^{\prime} \in \mathcal{H}^{\infty}$, and any~$a \in \mathcal{A}$, $o \in \mathcal{O}$,
\begin{align*}
&\left\lVert b(\cdot \mid h \circ (a,o)) - b(\cdot \mid h^{\prime} \circ (a,o)) \right\rVert_{TV}\\
&\qquad\qquad\qquad\leq (1-\rho) \left\lVert b(\cdot \mid h)-b(\cdot \mid h^{\prime}) \right\rVert_{TV}.
\end{align*}
\end{definition}
Filter stability may arise from either sufficiently noisy observations, which induce forgetting of the prior belief, or from mixing properties of the underlying state dynamics, whereby the influence of the initial state -- and hence the initial belief -- vanishes over time.
Without such forgetting, deriving finite-window approximation guarantees seems hopeless, since important information could reside arbitrarily far in the past.

Below, we state the sufficiency of our assumptions for filter stability; the proof is deferred to the appendix.
\begin{proposition}
\label{prop:uni-min-stab}
Suppose Assumptions~\ref{ass:uni-min} and~\ref{ass:uni-obs} hold for the POMDP $\mathcal{P}$. Then $\mathcal{P}$ satisfies filter stability as in Definition~\ref{def:filter-stab} with $\rho=S \alpha \beta$.
\end{proposition}
A crucial consequence of filter stability is that due to the exponential forgetting of history, we obtain guarantees for finite-window approximations. In particular, our model estimation guarantee relies on the following bound on the difference in transition probabilities between the superstate MDP $\mathcal{M}^m$ and the history-state MDP $\mathcal{M}^{\infty}$. The result is shown in \cite{anjarlekar2026scalable} as part of the proof of their Theorem~2. We restate it here for completeness.
\begin{lemma}
\label{lem:aux-transition}
Let $m \in \mathbb{N}$ and suppose that filter stability, as in Definition~\ref{def:filter-stab}, holds. Let $w \in \mathcal{H}^{\leq m}$ and $h \in \mathcal{H}^t$ for some $t \geq m$ with $h_{t-|w|+1 \,:\, t}=w$. Let further $a \in \mathcal{A}$ and $o \in \mathcal{O}$, and define $h^{\prime}=h \circ (a,o)$ and $w^{\prime}=w_{|w| - m + 2 \,:\, |w|} \circ (a,o)$. Then,
\begin{align*}
\left\vert \mathbb{P}^{m}(w^{\prime} \mid w,a) - \mathbb{P}^{\infty}(h^{\prime} \mid h,a) \right\vert \leq (1-\rho)^{m-1}.
\end{align*}
\end{lemma}

\subsection{Model Estimation Guarantee}

The following proposition characterizes the trajectory length $T$ required to achieve a desired model estimation error.
\begin{proposition}
\label{prop:est-prob}
Suppose Assumptions~\ref{ass:uni-min} and~\ref{ass:uni-obs} hold, and we estimate $\hat{\mathbb{P}}^m$ and $\hat{r}^m$ according to~(\ref{eq:sample-p}) and~(\ref{eq:sample-r}), respectively. Let $m \in \mathbb{N}$, $\delta>0$, and $\epsilon \geq 2(1-\rho)^{m-1}$ with $\rho = S \alpha \beta$. If $T$ is chosen as in (\ref{eqn:main-thm-T}), then with probability $1-\delta$, for all $a \in \mathcal{A}$, and $w,w^{\prime} \in \mathcal{H}^{\leq m}$, 
\begin{align}
\label{eqn:est-err-p}
\left\vert \mathbb{P}^m(w^{\prime} \mid w,a)-\hat{\mathbb{P}}^m(w^{\prime} \mid w,a) \right\vert \leq \epsilon,
\end{align}
and the reward estimates are exact, that is,
\begin{align}
\label{eqn:est-err-r}
\hat{r}^m(w,a)=r^m(w,a).
\end{align}
\end{proposition}
\begin{proof}
We first show~(\ref{eqn:est-err-p}); claim~(\ref{eqn:est-err-r}) is proven at the end.

At a high level, our proof proceeds as follows: For each tuple~$(w^{\prime},w,a) \in (\mathcal{H}^{\leq m})^2 \times \mathcal{A}$, the probability of misestimating $\mathbb{P}^m(w^{\prime} \mid w,a)$ can be written as an event in which a sum of indicator random variables describing the visitation of $(w^{\prime},w,a)$ deviates from its mean. Thus we aim to apply a concentration inequality to bound the probability of this event. Classical concentration bounds do not apply, as these indicator variables depend on the full action-observation history and are therefore generally not independent. However, we can show that the indicator variables are Lipschitz continuous functions of a hidden Markov chain whose underlying state process is contractive. Then, from~Theorem~1.2 of~\cite{kontorovich2008concentration} (restated as Theorem~\ref{thm:conc-ineq} in the appendix), we obtain a concentration bound for this setting. The model-based route is what enables this argument: applying Theorem~\ref{thm:conc-ineq} to each $(w^{\prime},w,a)$ tuple yields the $\mathcal{O}(\epsilon^{-2})$ rate. Model-free alternatives, such as finite-memory Q-learning~\cite{kara2023convergence}, instead track an iterate driven by history-dependent noise, for which the available finite-sample analyses yield the slower $\mathcal{O}(\epsilon^{-4})$ rate~\cite{cayci2024finite,anjarlekar2026scalable}.

In the following, $H \in \mathcal{H}^T$ denotes the random variable describing the sampled action-observation trajectory.
Let further $a \in \mathcal{A}$, and $w,w^{\prime} \in \mathcal{H}^{\leq m}$. 
Based on~$H$, we consider the following three undesirable events, corresponding to lack of visitation, and to over- and underestimation of transition probabilities, respectively:
\begin{align*}
E^1_{w,a} &\coloneqq \Big\{ \textstyle\sum_{t=1}^{T} \mathbf{1}\{ H_{t-|w|+1\,:\,t}=w \land a_t=a \} = 0 \Big\},\\
E^2_{w,w^{\prime},a} &\coloneqq \Big\{ \hat{\mathbb{P}}^m(w^{\prime} \mid w,a)-\mathbb{P}^m(w^{\prime} \mid w,a) \geq \epsilon \Big\},\\
E^3_{w,w^{\prime},a} &\coloneqq \Big\{ \mathbb{P}^m(w^{\prime} \mid w,a)-\hat{\mathbb{P}}^m(w^{\prime} \mid w,a) \geq \epsilon \Big\}.
\end{align*}
We aim to upper bound $P(E^1_{w,a} \cup E^2_{w,w^{\prime},a} \cup E^3_{w,w^{\prime},a})$ by bounding the probability of each event, followed by a union bound.

\underline{Upper bound on $P(E^2_{w,w^{\prime},a})$:} Without loss of generality, we assume $|w^{\prime}| = \min(m,|w| + 1)$, as otherwise $P(E^2_{w,w^{\prime},a})=0$. For each $t \in [T]$, we define indicator random variables for the occurrence of the transition from $w$ to $w^{\prime}$ under action $a$, and for the visitation of the pair $(w,a)$, respectively,
\begin{align*}
\phi^t_{w,w^{\prime},a}&\coloneqq \mathbf{1}\{ H_{t-|w^{\prime}|+2\,:\,t+1}=w^{\prime} \;\land\; H_{t-|w|+1\,:\,t}=w \;\land\; a_t=a \},\\
\psi^t_{w,a}&\coloneqq \mathbf{1}\{ H_{t-|w|+1\,:\,t}=w \land a_t=a \},
\end{align*}
as well as $\phi_{w,w^{\prime},a} \coloneqq \sum_{t=1}^T \phi^t_{w,w^{\prime},a}$ and $\psi_{w,a} \coloneqq \sum_{t=1}^T \psi^t_{w,a}$. Since $a \in \mathcal{A}$ and $w,w^{\prime} \in \mathcal{H}^{\leq m}$ have been fixed in the beginning, in the following, we omit the respective subscripts and use the simpler notations $\phi^t,\psi^t,\phi,\psi$. Next, let
\begin{align*}
\varphi^t \coloneqq \phi^t - \left( \mathbb{P}^m(w^{\prime} \mid w,a)+\epsilon \right) \psi^t.
\end{align*}
Then we have the following chain of implications,
\begin{align*}
&\mathbb{\hat{P}}^m(w^{\prime} \mid w,a)-\mathbb{P}^m(w^{\prime} \mid w,a) \geq \epsilon \\
&\qquad \Rightarrow\quad \frac{\phi}{\psi} \geq  \mathbb{P}^m(w^{\prime} \mid w,a) + \epsilon \\
&\qquad \Rightarrow\quad \sum_{t=1}^T \phi^t \geq \sum_{t=1}^T \left( \mathbb{P}^m(w^{\prime} \mid w,a) + \epsilon \right) \psi^t \\
&\qquad \Rightarrow\quad \sum_{t=1}^T \varphi^t \geq 0 \\
&\qquad \Rightarrow\quad \underbrace{\sum_{t=1}^T \left( \varphi^t - \mathbb{E}\left[ \varphi^t \right] \right) \geq -\sum_{t=1}^T \mathbb{E}\left[ \varphi^t \right]}_{(\star)}.
\end{align*}
The inequality $(\star)$ describes an event in which $\sum_{t=1}^T \varphi^t$ deviates from its mean. Here, this sum is a function of the hidden Markov chain given by the action-observation process $(o_t,a_t)_{t \geq 1}$, when choosing actions uniformly at random. Next, we show that even though the indicator variables $\varphi^t$ are not independent, the two conditions for applying the concentration bound in Theorem~\ref{thm:conc-ineq} hold.
\begin{enumerate}
\item \textbf{Contraction.} By Assumption~\ref{ass:uni-min}, transition probabilities for the Markov chain $(a_t,s_t)_{t \geq 1}$ are uniformly lower bounded by $\alpha / A$. Moreover, for two distributions $\nu,\nu^{\prime} \in \Delta(\mathcal{S} \times \mathcal{A})$, we can write
\begin{align*}
\left\lVert \nu-\nu^{\prime} \right\rVert_{TV} = 1-\sum_{(s,a) \in \mathcal{S} \times \mathcal{A}} \min(\nu(s,a),\nu^{\prime}(s,a)).
\end{align*}
Applying this identity and the above uniform lower bound to the definition of contraction in Theorem~\ref{thm:conc-ineq}, we obtain contraction with rate $\theta=1-\alpha S$.
\item \textbf{Lipschitz continuity.} Each change in the history affects at most $m+1$ of the indicator variables $\phi^t$, and likewise at most $m+1$ of the $\psi^t$. Since $\mathbb{P}^m(w^{\prime} \mid w,a)+\epsilon \leq 1+\epsilon$, the sum $\sum_{t=1}^T \varphi^t$ is therefore $L$-Lipschitz continuous as a function of $H$, in the sense defined in Theorem~\ref{thm:conc-ineq}, with $L=(2+\epsilon)(m+1)$.
\end{enumerate}
To derive a bound from Theorem~\ref{thm:conc-ineq}, we need a lower bound on the right-hand side of $(\star)$, which we provide next.

\medskip
\noindent\textbf{Claim.} For every $t$ with $|w| \leq t \leq T-1$, $\mathbb{E}\left[ \varphi^t \right] \leq -\frac{\epsilon\beta^m}{2A^m}$.
\begin{proof}[Proof of Claim]
Let $(a,o) \in \mathcal{A} \times \mathcal{O}$ be the action-observation pair with $w^{\prime} = w_{|w| - m + 2 \,:\, |w|} \circ (a,o)$, and for a full history $h \in \mathcal{H}^t$ write $h^{\prime} \coloneqq h \circ (a,o)$. Conditioning on the full history, the law of total probability gives
\begin{align}
\label{eqn:phi-mixture}
\frac{P(\phi^t=1)}{P(\psi^t=1)} = \sum_{h} P(H_{1:t}=h \mid \psi^t=1) \; \mathbb{P}^{\infty}(h^{\prime} \mid h,a),
\end{align}
where the sum ranges over all $h \in \mathcal{H}^t$ whose last $|w|$ steps agree with $w$: conditioned on such an $h$ and on $a_t=a$, we have $\phi^t=1$ precisely if the observation received at time $t+1$ equals $o$. The right-hand side of~(\ref{eqn:phi-mixture}) is a convex combination whose weights depend on $t$ and on the prior $\mu$. These weights need not be known, however, since Lemma~\ref{lem:aux-transition} bounds the difference $\vert \mathbb{P}^{\infty}(h^{\prime} \mid h,a) - \mathbb{P}^m(w^{\prime} \mid w,a) \vert$ by $(1-\rho)^{m-1}$ \emph{uniformly} over all full histories $h$ ending with $w$, so that the same bound applies to the convex combination. Using the definition of $\varphi^t$ and the fact that $\phi^t$ and $\psi^t$ are indicator random variables, we thus obtain
\begin{align*}
\mathbb{E}\left[ \varphi^t \right] &= \left( \frac{P(\phi^t=1)}{P(\psi^t=1)} - \mathbb{P}^m(w^{\prime} \mid w,a) - \epsilon \right) P(\psi^t=1) \\
&\overset{(a)}{\leq} \left( (1-\rho)^{m-1} - \epsilon \right) P(\psi^t=1) \;\overset{(b)}{\leq}\; -\frac{\epsilon\beta^m}{2A^m},
\end{align*}
where (a) is by Proposition~\ref{prop:uni-min-stab} and the above, and (b) uses $\epsilon \geq 2(1-\rho)^{m-1}$ as well as Assumption~\ref{ass:uni-obs}, which implies $P(\psi^t=1) \geq \frac{\beta^m}{A^m}$.
\end{proof}

For $t<|w|$ the window $w$ cannot have occurred, so that $\phi^t=\psi^t=0$ and hence $\mathbb{E}\left[ \varphi^t \right]=0$; for $t=T$ we have $\phi^T=0$, since the trajectory contains no observation at time $T+1$, and thus $\mathbb{E}\left[ \varphi^T \right] \leq 0$. Together with the claim, this gives
\begin{align*}
-\sum_{t=1}^{T} \mathbb{E}\left[ \varphi^t \right] \;\geq\; (T-m)\,\frac{\epsilon\beta^m}{2A^m} \;\geq\; \frac{T \epsilon\beta^m}{4A^m},
\end{align*}
where the last step uses $T \geq 2m$, which is implied by~(\ref{eqn:main-thm-T}). Theorem~\ref{thm:conc-ineq} now yields that
\begin{align*}
P(E^2_{w,w^{\prime},a}) \leq 2\exp \left( -\frac{\alpha^2 S^2 T \epsilon^2 \beta^{2m}}{32(2+\epsilon)^2(m+1)^2 A^{2m}} \right).
\end{align*}
With our choice of $T$, this implies $P(E^2_{w,w^{\prime},a}) \leq \frac{\delta}{24\,A^{2m+1}O^{2m}}$. By a symmetric argument, for the same choice of~$T$, we get $P(E^3_{w,w^{\prime},a}) \leq \frac{\delta}{24\,A^{2m+1}O^{2m}}$.

\underline{Upper bound on~$P(E^1_{w,a})$:} Similar to above, we aim to apply Theorem~\ref{thm:conc-ineq} for bounding
\begin{align}
\label{eqn:e1-bound}
P(E^1_{w,a}) \leq P \left( \psi - \mathbb{E}\left[ \psi \right] \leq -\mathbb{E} \left[ \psi \right] \right).
\end{align}
For this we observe that by linearity of expectation,
\begin{align*}
\mathbb{E} \left[ \psi \right] = \sum_{t=1}^T P(H_{t-|w|+1 \,:\, t}=w \land a_t=a),
\end{align*}
and each summand with $t \geq |w|$ is, as argued above, lower bounded by $\frac{\beta^m}{A^m}$, implying $\mathbb{E} \left[ \psi \right] \geq (T-m+1)\frac{\beta^m}{A^m} \geq \frac{T \beta^m}{2A^m}$.
Plugging this into (\ref{eqn:e1-bound}), and applying Theorem~\ref{thm:conc-ineq} with $L=m+1$, we get
\begin{align*}
P(E^1_{w,a}) \leq 2\exp \left( -\frac{\alpha^2 S^2 T \beta^{2m}}{8(m+1)^2 A^{2m}} \right).
\end{align*}
Hence our choice of $T$ suffices to have $P(E^1_{w,a}) \leq \frac{\delta}{24\,A^{2m+1}O^{2m}}$.

To conclude the proof, note that if none of $E^1_{w,a}$, $E^2_{w,w^{\prime},a}$, or $E^3_{w,w^{\prime},a}$ occurs for any $w,w^{\prime} \in \mathcal{H}^{\leq m}$ and $a \in \mathcal{A}$, then we obtain the desired bound. By a union bound, this happens with probability at least
\begin{align*}
&1 - \sum_{w,w^{\prime} \in \mathcal{H}^{\leq m}} \sum_{a \in \mathcal{A}} \left( P(E^1_{w,a}) + P(E^2_{w,w^{\prime},a}) + P(E^3_{w,w^{\prime},a}) \right) \\
\geq\;\, &1 - \left\vert \mathcal{H}^{\leq m} \right\vert^2 \cdot \left\vert \mathcal{A} \right\vert \cdot \frac{3\delta}{24\,A^{2m+1}O^{2m}} \\
=\;\, &1 - \frac{\delta}{2}
\end{align*}
where we have used $|\mathcal{H}^{\leq m}| \leq 2(AO)^m$. On the same event, in particular no $E^1_{(o),a}$ with $o \in \mathcal{O}$ occurs, that is, every observation-action pair is visited at least once. For any such visit at time $\hat{t}$ we have $r_{\hat{t}}=r(o_{\hat{t}},a)$, so that~(\ref{eq:sample-r}) returns $\hat{r}^m(w,a)=r(\tilde{o}_{|w|},a)=r^m(w,a)$, which proves~(\ref{eqn:est-err-r}). Thus both claims hold with probability at least $1-\delta$.
\end{proof}

\subsection{Guarantee for Planning under Approximation}

Having established the model estimation guarantee, we are now ready to prove our main result, Theorem~\ref{thm:main}.
\begin{proof}[Proof of Theorem~\ref{thm:main}]
First, we introduce the value function of the superstate MDP $\mathcal{M}^m$. For a policy $\pi^m \in \Pi^m$, let
\begin{align*}
V^m(\pi^m) \coloneqq \mathbb{E}_{\pi^m,s_1 \sim \mu} \left[ \sum_{t=1}^{\infty} \gamma^{t-1} r^m(w_t,a_t) \right]
\end{align*}
where $w_t \in \mathcal{H}^{\leq m}$ is the superstate at time $t$ and the expectation is taken over trajectories sampled according to $\mathbb{P}^m$. We also let $V^m_{\star} \coloneqq \max_{\pi^m \in \Pi^m} V^m(\pi^m)$. Moreover, we define the estimated superstate MDP $\hat{\mathcal{M}}^m$ with value function
\begin{align*}
\hat{V}^m(\pi^m) \coloneqq \mathbb{E}_{\pi^m,s_1 \sim \mu} \left[ \sum_{t=1}^{\infty} \gamma^{t-1} \hat{r}^m(w_t,a_t) \right]
\end{align*}
where the expectation is taken over trajectories sampled according to $\hat{\mathbb{P}}^m$. Let $\hat{V}^m_{\star} \coloneqq \max_{\pi^m \in \Pi^m} \hat{V}^m(\pi^m)$.

We now decompose our error into
\begin{align*}
V^{\star} - V(\pi^m) &\leq \underbrace{\left\vert V^{\star} - V^m_{\star} \right\vert}_{(a)} + \underbrace{\left\vert V^m_{\star} - \hat{V}^m_{\star} \right\vert}_{(b)} + \underbrace{\left\vert \hat{V}^m_{\star} - V^m(\pi^m) \right\vert}_{(c)}\\
&\qquad\qquad\qquad\qquad\qquad\;\; + \underbrace{\left\vert V^m(\pi^m)-V(\pi^m) \right\vert}_{(d)}.
\end{align*}
Here, (a) and (d) are errors due to finite-window approximation which under filter stability are shown to be bounded in~\cite{anjarlekar2026scalable} by $2(1-\rho)^{m-1} / (1-\gamma)^2$.
Moreover, (b) is the error due to model estimation, and~(c) is the value iteration error. We will now proceed with bounding~(b) and~(c).

In the following, we condition on the event of Proposition~\ref{prop:est-prob}, on which the transition estimates satisfy $\vert \mathbb{P}^m-\hat{\mathbb{P}}^m \vert \leq \epsilon$ while the reward estimates are exact.
Then for (b), by the so-called simulation lemma (see~\cite{kearns2002near}, Lemma~4), we have
\begin{equation}
\label{eqn:main-b}
\begin{aligned}
(b) &\leq \sup_{\pi^m \in \Pi^m} \left\vert V^m(\pi^m)-\hat{V}^m(\pi^m) \right\vert \leq \frac{\gamma \epsilon}{(1-\gamma)^2}.
\end{aligned}
\end{equation}
To bound (c), we further decompose the error into
\begin{align*}
\left\vert \hat{V}^m_{\star} - V^m(\pi^m) \right\vert \leq \underbrace{\left\vert \hat{V}^m_{\star} - \hat V^m(\pi^m) \right\vert}_{(c_1)} + \underbrace{\left\vert \hat V^m(\pi^m) - V^m(\pi^m) \right\vert}_{(c_2)}.
\end{align*}
Then (c$_1$) is the standard value iteration error \cite{bertsekas2012dynamic} which is known to be bounded by $\frac{2\gamma^K}{1-\gamma}$. By an argument analogous to~(\ref{eqn:main-b}), we have $(c_2) \leq \frac{\gamma \epsilon}{(1-\gamma)^2}$.

Combining the above bounds, we obtain
\begin{align*}
V^{\star} - V(\pi^m)
\leq \frac{2\gamma^K}{1-\gamma} + \frac{2\gamma\epsilon + 4(1-\rho)^{m-1}}{(1-\gamma)^2},
\end{align*}
from which the desired bound follows with our choice of~$K$ and $\gamma \leq 1$.
\end{proof}

\section{Simulation}

We illustrate the convergence of Algorithm~\ref{alg:main} on a POMDP environment, \emph{Probe}. The environment is deliberately simple, allowing us to verify our assumptions, yet sufficiently rich so that increasing the window size $m$ leads to improved optimal policies.

The environment is defined as follows. A hidden state $s \in \mathcal{S}=\{s_1,s_2\}$ is initialized uniformly at random. The agent selects actions from $\mathcal{A}=\{p,a_1,a_2\}$, and observations take values in $\mathcal{O}=\{o_1,o_2\}$. The action $p$ (probing) reveals the true state with probability $0.95$, yielding observation $o_i$ when the underlying state is $s_i$, and produces a uniformly random observation otherwise. The actions $a_1$ and $a_2$ yield an uninformative observation, returning a uniformly random observation with probability $0.95$, and the observation corresponding to the true state otherwise. The reward function is defined as $r(o_1,a_1)=r(o_2,a_2)=1$ and $r(o_1,a_2)=r(o_2,a_1)=-1$, while probing yields no reward, $r(o_1,p)=r(o_2,p)=0$. State transitions do not depend on the chosen actions: the state remains unchanged with probability $0.95$, and switches otherwise.

Note that \emph{Probe} satisfies Assumption~\ref{ass:uni-min} with $\alpha=0.05$. Since an action influences the observation received at the next time step, \emph{Probe} has an observation kernel of the more general form $\mathbb{O}(\cdot \mid s_{t+1},a_t)$, which satisfies the condition of Assumption~\ref{ass:uni-obs} with $\beta=0.025$.\footnote{All our results extend to this more general class of observation kernels.} Moreover, for small window sizes $m$, the optimal policy is simple: probe once, then choose $a_1$ or $a_2$ according to the observed outcome for $m$ steps, before probing again. 

\begin{figure}[t]
\centering
\includegraphics[width=\linewidth]{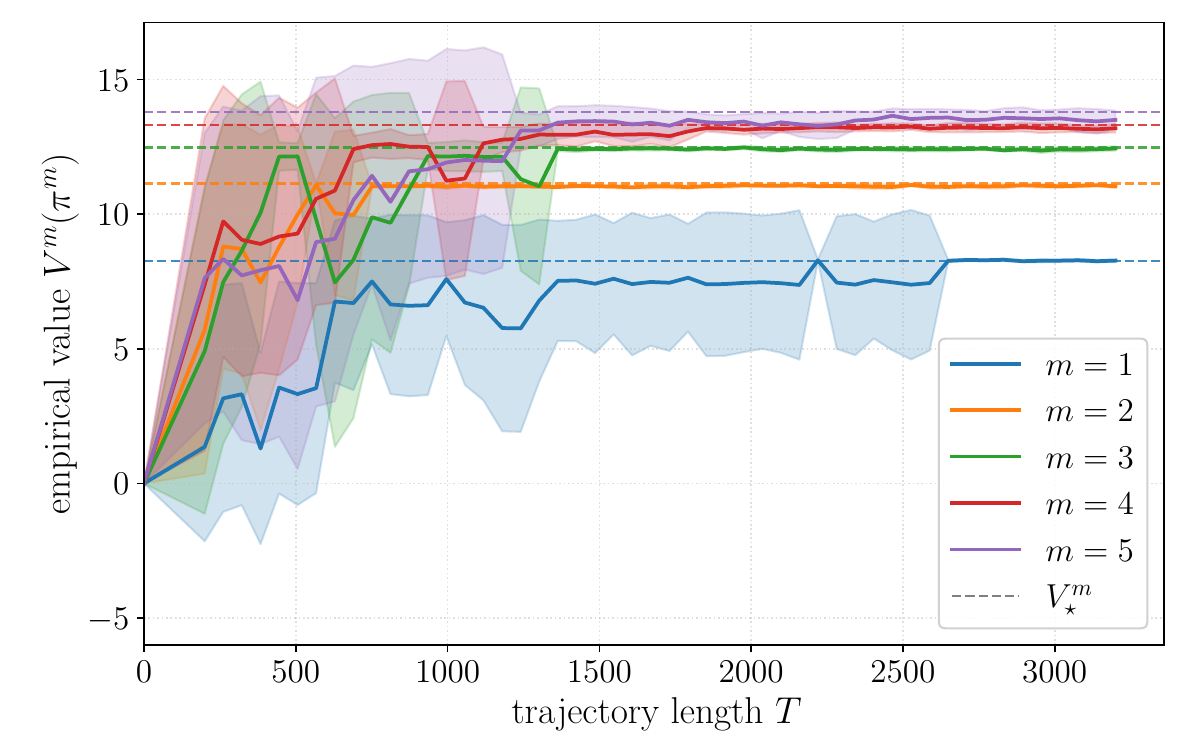}
\caption{Empirical convergence of $V^m(\pi^m)$ towards $V^m_\star$ for $m \in \left\{ 1,\dots,5 \right\}$ when running Algorithm~\ref{alg:main} in the \emph{Probe} environment.}
\label{fig:plot}
\end{figure}

Figure~\ref{fig:plot} shows the optimal values $V^m_\star$ of the superstate MDPs, which increase with $m$. We plot the empirical value $V^m(\pi^m)$ that is obtained by running Algorithm~\ref{alg:main} with a fixed number of value iterations $K=50$ and various trajectory lengths $T$ for the model estimation subroutine. The plot shows the average and standard deviation across 10 runs for each $m \in \left\{ 1,\dots,5 \right\}$. We observe that once $T$ is sufficiently large to estimate the superstate model well, the value achieved by the obtained policies converges to the respective~$V^m_\star$.

\section{Conclusion}

We studied model-based learning of finite-window policies in tabular POMDPs via a superstate MDP reduction. Our main result shows that under our Assumptions~\ref{ass:uni-min} and~\ref{ass:uni-obs}, the superstate MDP can be estimated from a single POMDP trajectory with tight sample complexity guarantees. Combined with value iteration, this yields a principled algorithm for learning near-optimal finite-window policies.

An important direction for future work is to investigate whether our assumptions can be relaxed. Moreover, it would be valuable to explore extensions of our analysis and empirical evaluation to more general POMDP environments, such as large or continuous state and action spaces through function approximation.

\section*{ACKNOWLEDGMENT}

This work was supported by the Swiss National Science Foundation under grant~200020\_207984.

\section*{APPENDIX}

We prove that Assumptions~\ref{ass:uni-min} and~\ref{ass:uni-obs} imply filter stability.
\begin{proof}[Proof of Proposition~\ref{prop:uni-min-stab}]
When taking action $a \in \mathcal{A}$ and observing $o \in \mathcal{O}$, for any $s^{\prime} \in \mathcal{S}$, the updated belief is
\begin{align*}
b(s^{\prime} \mid h \circ (a,o)) = \frac{\sum_{s \in \mathcal{S}} b(s \mid h)\; \mathbb{P}(s^{\prime} \mid s,a)\; \mathbb{O}(o \mid s^{\prime})}{\sum_{s^{\prime\prime} \in \mathcal{S}} \sum_{s \in \mathcal{S}} b(s \mid h)\; \mathbb{P}(s^{\prime\prime} \mid s,a)\; \mathbb{O}(o \mid s^{\prime\prime})}.
\end{align*}
We can lower bound the numerator by
\begin{align*}
\sum_{s \in \mathcal{S}} b(s \mid h)\; \mathbb{P}(s^{\prime} \mid s,a) \; \mathbb{O}(o \mid s^{\prime}) &\geq \alpha \beta \sum_{s \in \mathcal{S}} b(s \mid h) \geq \alpha \beta
\end{align*}
where we have used Assumptions~\ref{ass:uni-min} and~\ref{ass:uni-obs}. Since the denominator is upper bounded by $1$, we obtain for the uniform distribution $\nu \in \Delta(\mathcal{S})$ with $\nu(s)=1 / S$ for each $s \in \mathcal{S}$, that
\begin{align*}
b(s \mid h \circ (a,o)) \geq S \alpha \beta \, \nu(s)
\end{align*}
for any $s \in \mathcal{S}$. This is a minorization (Doeblin) condition: every one-step belief update is bounded below by the fixed measure $S \alpha \beta \, \nu$, uniformly over the history $h$, the action $a$, and the observation $o$. Then Lemma~5.2 of~\cite{van2008hidden} applies and yields the contraction of Definition~\ref{def:filter-stab} with $\rho=S \alpha \beta$.
\end{proof}

For completeness, we restate the following concentration inequality for functions over hidden Markov chains, which is a key tool in our proof of Theorem~\ref{thm:main}. For reference, see~\cite{kontorovich2008concentration}, Theorem~1.2 and Theorem~7.1.
\begin{theorem}
\label{thm:conc-ineq}
Let $(X_1,\dots,X_n)$ and $(Y_1,\dots,Y_n)$ be sequences of random variables with each $X_i \in \mathcal{X}$ and $Y_i \in \mathcal{Y}$ where $\mathcal{X}$ and $\mathcal{Y}$ are countable sets. Suppose $(Y_1,\dots,Y_n)$ forms a hidden Markov chain with underlying chain $(X_1,\dots,X_n)$.
Further, let $\varphi:\mathcal{Y}^n \to \mathbb{R}$ and denote $\mathbb{E}\varphi \coloneqq \mathbb{E}_{(Y_1,\dots,Y_n)}[\varphi(Y_1,\dots,Y_n)]$. Suppose the following two conditions hold:
\begin{enumerate}
\item\label{item:contr} \textbf{Contraction of underlying Markov chain.} There exists $0 < \theta < 1$ such that for all $i \in [n-1]$,
\begin{align*}
\hspace{-2em}\max_{x_i^{\prime},x_i^{\prime\prime} \in \mathcal{X}} \left\lVert P(X_{i+1} \mid X_i=x_i^{\prime})-P(X_{i+1} \mid X_i=x_i^{\prime\prime}) \right\rVert_{TV} \leq \theta.
\end{align*}
\item\label{item:lip} \textbf{Lipschitz continuity.} There exists $L>0$ such that for all $y,y^{\prime} \in \mathcal{Y}^n$,
\begin{align*}
\left\vert \varphi(y)-\varphi(y^{\prime}) \right\vert \leq L \cdot d_H(y,y^{\prime}),
\end{align*}
where $d_H(y,y^{\prime}) \coloneqq \sum_{i \in [n]} \mathbf{1}\{ y_i \not= y_i^{\prime} \}$.
\end{enumerate}
Then for any $t > 0$,
\begin{align*}
P_{(Y_1,\dots,Y_n)}\Big( \left\vert \varphi(Y_1,\dots,Y_n) - \mathbb{E}\varphi \right\vert \geq t \Big)
\leq 2 \exp \left( - \frac{(1-\theta)^2 t^2}{2nL^2} \right).
\end{align*}
\end{theorem}

\bibliographystyle{IEEEtran}
\balance
\bibliography{refs}

\end{document}